%% file: main.tex
\documentclass[pdflatex,sn-mathphys-num]{sn-jnl}

\usepackage{graphicx}
\usepackage{amsmath,amssymb,amsfonts}
\usepackage{array}

\graphicspath{{figures/}{./}}
\newcommand{\vect}[1]{\boldsymbol{#1}}
\newcommand{\mat}[1]{\mathbf{#1}}

\newcommand{\norm}[1]{\left\lVert #1\right\rVert}

\theoremstyle{thmstyleone}
\newtheorem{proposition}{Proposition}

\begin{document}

\title[Robot Aware Design of Object Specific Passive Grippers]{Robot Aware Computational Design of Object Specific Passive Grippers for Additive Manufacturing}

\author*[1,2]{\fnm{Abdullah Yahya Abdullah} \sur{Omaisan}}\email{abuod0@gmail.com}
\equalcont{These authors contributed equally to this work.}

\author[1,2]{\fnm{Ibrahim Sheikh} \sur{Mohamed}}\email{ibrahim1sheikh1@gmail.com}
\equalcont{These authors contributed equally to this work.}

\affil[1]{\orgname{Independent Researchers},
\orgaddress{\city{Riyadh}, \country{Saudi Arabia}}}

\affil[2]{\orgname{QSS AI and Robotics Lab},
\orgaddress{\city{Riyadh}, \country{Saudi Arabia}}}

\abstract{
This paper presents an end-to-end computational pipeline that converts a selected object mesh, a measured object state, and a selected six-axis robot into an object-specific, unactuated, additively manufacturable gripper.  The method couples exact-mesh RGB-D/ICP pose registration, deterministic surface-contact sampling, uncertainty-aware wrench screening, selection among six passive capture mechanisms, object-conformal surface synthesis, full-orientation robot inverse kinematics, a swept-volume-aware manufacturing domain, directional fused-deposition finite-element screening, and constrained three-dimensional SIMP topology optimization.  Unlike workflows that treat grasp selection, tool geometry, motion, and structural design as separate problems, every exported design is bound to the source mesh, object pose, robot flange, contact set, and insertion hypothesis by a traceable design identifier.  We derive the implemented registration, contact, fit-tolerance, finite-element, and density-optimization equations and prove three properties of the numerical construction: nodal load preservation, monotonic compliance sensitivity under SIMP interpolation, and voxel-domain containment after topology post-processing.  Four archived object-specific attempts---a rabbit, camera flange, 3DBenchy, and faceted bust---meet the nominal fit, uncertain-wrench, runtime-sweep, and baseline/post-topology FEA gates.  A deliberately enlarged $\pm3$~mm, $\pm5^\circ$ pose stress check differentiates the designs, retaining 29--134 of 160 simulated trials. Their reconstructed topologies retain 92.0--97.6\% of the FE domain because functional regions are protected.  Archived robot photographs show the corresponding printed assemblies qualitatively, while nominal material properties and absent coupon-calibrated, instrumented tests keep all four at digital-screening status rather than operational release.
}


\maketitle

\input{1_Introduction}
\input{2_Related_Works}
\input{3_Methods}
\input{4_Results}
\input{5_conclusion}

\backmatter

\section*{Declarations}

\noindent\textbf{Ethics approval and consent to participate.} Not applicable. This computational engineering study did not involve human participants, human data, animals, or biological materials.

\begingroup
\renewcommand{\bibfont}{\fontsize{7}{8}\selectfont}
\setlength{\bibsep}{0pt}
\bibliography{References}
\endgroup

\end{document}

%% file: 1_Introduction.tex
\section{Introduction}
\label{sec:introduction}

Object-specific grippers improve repeatability but are commonly produced through manual, experience-dependent iteration \cite{Honarpardaz2017}. Additive manufacturing lowers fabrication cost; it does not determine where the tool should contact an object, which passive retention mechanism is appropriate, whether the robot can insert it, or whether the printed load path is safe. Passive grippers intensify this coupling because robot motion and object geometry must provide insertion, locking, carrying, and release without a gripper-side actuator \cite{Kodnongbua2022}.

The proposed framework converts a millimetre-scale object mesh, object pose and mass, robot profile, manufacturing profile, and pick/place task into a traceable design package. The package contains ranked contact evidence, a connected watertight preform with flange interface, an insertion--lock--release program, robustness records, finite-element reports, and an optional topology-optimized STL. Every artifact is bound to the source mesh and design hypothesis; physical evidence remains a separate operational-release requirement. Figure~\ref{fig:pipeline} summarizes the process.

The principal contributions are:
\begin{itemize}
    \item an implemented robot-aware pipeline joining selected-mesh ICP pose estimation, uncertainty-aware contact ranking, six passive capture families, conformal synthesis, full-orientation insertion motion, FEA, and topology optimization;
    \item a directional printed-material model with distributed loads, protected SIMP domains, post-optimization reconstruction, and proofs of force preservation, compliance monotonicity, and voxel-domain containment; and
    \item four traceable, object-specific design records that connect mechanism images to fit, wrench, runtime-sweep, structural, topology, and operational-release evidence.
\end{itemize}

%% file: 2_Related_Works.tex
\section{Related Work}
\label{sec:related}

Automated finger design combines grasp analysis, contact geometry, accessibility, and manufacturability \cite{Honarpardaz2017}. Song \emph{et al.} optimized printable contact primitives with enlarged contact area \cite{Song2018}, while Lim and Pham synthesized object-conformal caging pads by Boolean geometry \cite{Lim2024}. Fit2Form instead learns parallel-jaw finger volumes from predicted grasp fitness \cite{Ha2021}; Dex-Net learns robust grasp quality from large synthetic datasets \cite{Mahler2017}. These approaches automate fingers or grasp proposals but do not produce a flange-connected passive body and its insertion program. The closest passive baseline is Kodnongbua \emph{et al.}, who jointly optimize a printable gripper skeleton and trajectory and validate the tools physically \cite{Kodnongbua2022}. Our contribution is complementary: an auditable engineering pipeline that selects an explicit passive mechanism, binds it to a robot/flange/camera configuration, expands ideal contacts into conformal regions, and carries the design through uncertainty and structural screening.

Rigid CAD-to-range alignment follows the iterative closest point formulation of Besl and McKay \cite{Besl1992}. Here registration is not a detached visualization step: the accepted full pose and exact selected-mesh identity become trajectory inputs.

Geometric restraint, force closure, and wrench-space quality provide the theoretical basis for contact ranking \cite{Nguyen1988,Bicchi1995,Ferrari1992}. Because nominal quality need not survive pose error, Weisz and Allen sampled uncertainty and re-ranked grasps in contact-wrench space \cite{Weisz2012}. We retain this separation: a nonnegative wrench residual cheaply screens ideal three-contact sets, whereas a second deterministic audit perturbs friction, center of mass, gravity direction, local translation, and rotation. Neither score is presented as a measured grasp probability because the manufactured body also relies on distributed surfaces and positive retention.

For structural refinement, density-based SIMP and finite-element sensitivity analysis provide a compact minimum-compliance formulation \cite{Sigmund2001,LiuTovar2014}; morphology-based filters help impose a realizable length scale \cite{Sigmund2007}. Additive manufacture further introduces anisotropy, feature-size, and post-processing constraints \cite{Zhu2021}, while direction-dependent FDM properties are well documented \cite{Ahn2002}. Topology optimization has been applied to lightweight adaptive fingers \cite{Sun2022} and physically validated multi-material soft grippers \cite{Pinskier2024}. Allouzi \emph{et al.} applied topology optimization to a food-serving robot to reduce component mass and energy demand \cite{Allouzi2023}. Here the domain is more restricted: robot motion, conformal fit, flange geometry, camera clearance, positive stops, minimum walls, and an outer skin are fixed before material can be removed.

%% file: 3_Methods.tex
\section{Method}
\label{sec:method}

\begin{figure}[t]
    \centering
    \includegraphics[width=0.98\textwidth]{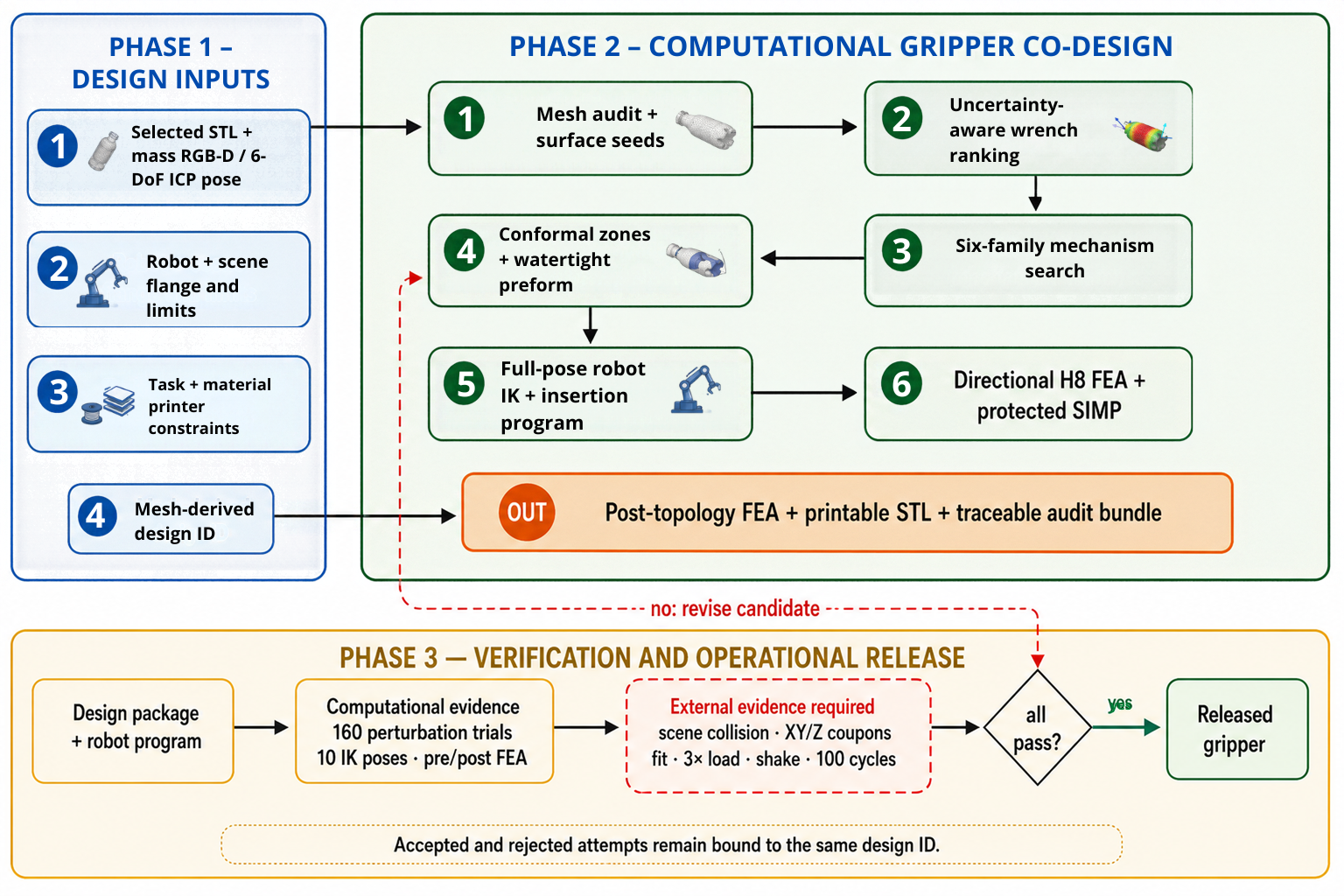}
    \caption{Proposed computational workflow. Phase 1 registers the inputs under a traceable identifier. Phase 2 co-designs contacts, passive geometry, robot motion, and structure. Phase 3 separates computational checks from the external evidence required for operational release; failed gates return to mechanism and geometry search.}
    \label{fig:pipeline}
\end{figure}

\subsection{Formulation, Mesh Audit, and Contacts}

The design input and output are
\begin{equation}
 \mathcal I=(\mathcal M,\mat T_o^w,m,\mathcal R,\mathcal P,\mathcal Q),
 \qquad \mathcal O=(\Omega_g,\tau),
 \label{eq:input}
\end{equation}
where $\mathcal M$ is the object mesh, $\mat T_o^w$ its pose, $m$ its mass, $\mathcal R$ the robot/flange/camera profile, $\mathcal P$ the printer/material profile, and $\mathcal Q$ the task. The solid $\Omega_g$ and joint trajectory $\tau$ must be collision- and keep-out-free, connected, printable, IK-feasible, robust to the declared fit error, and below stress and displacement limits. These are ordered gates, not terms that can compensate for one another. A SHA-256 identifier over the mesh, pose, robot, contacts, family, motion, clearances, and voxel settings is attached to every artifact.

The mesh audit records topology, bounds, area, volume, and watertightness, then rejects floor-adjacent facets. From 600 deterministic surface seeds, the implementation proposes 1,200 three-seed sets and rejects inadequate spacing, degenerate projected area, weak normal opposition, and blocked terminal paths. For contact $i$, an eight-sided Coulomb cone is
\begin{align}
 \vect f_{ij}&=\frac{-\vect n_i+\mu(\cos\theta_j\vect t_{i1}+\sin\theta_j\vect t_{i2})}
 {\norm{-\vect n_i+\mu(\cos\theta_j\vect t_{i1}+\sin\theta_j\vect t_{i2})}},\\
 \vect w_{ij}&=[\vect f_{ij}^{\mathsf T},((\vect p_i-\vect c)\times\vect f_{ij})^{\mathsf T}/L]^{\mathsf T},
 \quad \eta=\min_{\vect\lambda\geq0}\norm{\mat W\vect\lambda-[-\widehat{\vect g};\vect0]}_2,
 \label{eq:nnls}
\end{align}
with $\mu=0.5$, $\theta_j=2\pi j/8$, center $\vect c$, and maximum extent $L$. The residual $\eta\leq0.22$ screens support but is not a force-closure certificate. Candidate score $S=100\vect\alpha^{\mathsf T}\vect q$ uses the implemented weights
$\vect\alpha=(.24,.08,.06,.08,.06,.04,.08,.16,.04,.16)$ for equilibrium, spacing, area, opposition, escape, connector length, support, compactness, outward direction, and local height. Up to 12 candidates continue. A stricter 75-trial audit spans three friction values, five center-of-mass offsets, and five gravity directions; at least 95\% must satisfy residual and coefficient-concentration limits.

\subsection{Mechanism and Geometry Synthesis}

Explainable mesh descriptors select an initial family: hollow exterior $\rightarrow$ rim cradle; fragmented topology $\rightarrow$ topology freeform; persistent waist $\rightarrow$ C-wrench; reliable opening $\rightarrow$ insert/key; elongated accessible body $\rightarrow$ fork/hook; otherwise $\rightarrow$ conformal cradle. Family selection remains revisable during manufacturing and motion checks (Fig.~\ref{fig:strategies}).

\begin{figure}[t]
 \centering
 \includegraphics[width=0.98\textwidth]{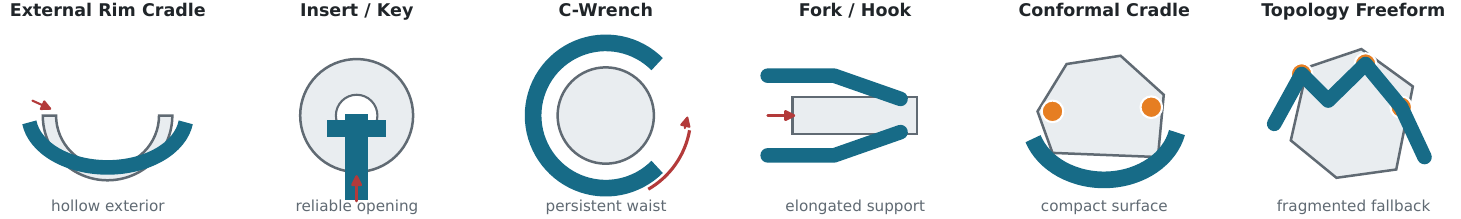}
 \caption{Implemented passive capture families and their insertion or rotation principle.}
 \label{fig:strategies}
\end{figure}

Each accepted anchor is expanded into a projected conformal patch. For projected vertices $\vect v_k$ and local normals $\vect n_k$,
\begin{equation}
 \vect v_k^{\rm in}=\vect v_k+c_f\vect n_k,
 \qquad \vect v_k^{\rm out}=\vect v_k+(c_f+t_p)\vect n_k,
 \label{eq:patch}
\end{equation}
where $c_f=0.18$~mm and $t_p=4$~mm; the rigid export reserves at least 1~mm for a compliant liner. Patches join named family features, a four-hole ISO 9409-1 mount, load paths, and a camera-safe neck. Voxel Boolean operations subtract the object, insertion sweep, mount holes, and keep-outs. Connectivity, retained fit zones, minimum wall, compactness, watertightness, and camera clearance are checked before surface extraction and protected smoothing.

\subsection{Selected-Mesh Registration, Motion, and Fit Robustness}

For a physical pick, the selected STL and its SHA-256 digest are authoritative. The implementation uniformly samples its centered surface into $\mathcal P_M$ and segments the aligned depth frame into $\mathcal P_D$. Forward kinematics and calibrated flange--camera and world--base transforms provide a prior. Seven local yaw seeds ($-18,-10,-5,0,5,10,18^\circ$) initialize robust six-DoF point-to-plane ICP \cite{Besl1992} at voxel/correspondence scales $(5/22,2.5/11,1.2/5.5)$~mm:
\begin{equation}
 (\mat R^*,\vect t^*)=\arg\min_{\mat R\in SO(3),\vect t}
 \sum_i \rho_{\rm T}\!\left(\big[(\mat R\vect p_i+\vect t-\vect q_{\pi(i)})^{\mathsf T}\vect n_{\pi(i)}\big]^2\right),
 \label{eq:icp}
\end{equation}
where $\pi(i)$ is the closest observed point and $\rho_{\rm T}$ is Tukey's robust loss. The best seed maximizes fitness minus 16 times RMSE and is evaluated at 6~mm. Registration must achieve fitness $\geq0.30$, RMSE $\leq4$~mm, and corrections $\leq100$~mm and $\leq40^\circ$. Its full orientation is recovered from
$\mat T_o^w=\mat T_b^w\mat T_f^b\mat T_c^f\mat T_o^c$; changing either the selected path or digest invalidates the pose.

Candidates seed insertion side, angle, rotation, and clearance hypotheses. Object-frame waypoint offsets $\vect r_k^o$, tool rotations $\mat R_{g,k}^o$, and insertion directions $\vect d_k^o$ consume the complete registered orientation:
\begin{equation}
 \vect x_k^w=\vect t_o^w+\mat R_o^w\vect r_k^o,\quad
 \mat R_k^w=\mat R_o^w\mat R_{g,k}^o,\quad
 \vect d_k^w=\mat R_o^w\vect d_k^o.
 \label{eq:fullpose}
\end{equation}
Thus roll and pitch cannot be silently reduced to yaw for a physical plan. At waypoint $k$, inverse kinematics minimizes
\begin{equation}
 \vect q_k^*=\arg\min_{\vect q\in\mathcal Q_{\mathcal R}}
 \big(\norm{\vect x(\vect q)-\vect x_k}_2^2+\beta d_R(\mat R(\vect q),\mat R_k)^2\big)
 \label{eq:ik}
\end{equation}
under the robot limits. The stored program covers home, pre-insert, insert, lock, lift, transfer, place, unlock, and retract. The scene records the ICP pose identifier, selected-mesh path/digest, fitness, and RMSE; physical validation rejects a trajectory without this exact binding. The archived runs then record a finite 233-frame browser runtime sweep of the table, robot, installed tool stack, generated mesh, and attached object.

For translation $\vect t$ and small rotation $\vect\phi$, contact displacement and retained-zone margin are
\begin{align}
 \vect d_i&=\vect t+\vect\phi\times(\vect p_i-\vect c),\\
 \gamma_i&=\min\!\left(a-|\vect d_i\!\cdot\!\vect n_i|,
 0.42r_i-\norm{\vect d_i-\vect n_i(\vect d_i\!\cdot\!\vect n_i)}_2\right),
 \label{eq:fitmargin}
\end{align}
where $a$ combines clearance, preview elasticity, and liner compression. A zone survives when $\gamma_i\geq0$. The nominal deterministic audit uses 160 candidate-ID-seeded trials over per-axis translation $\pm0.8$~mm and rotation $\pm1^\circ$, requiring 95\% family-adequate retention.

\begin{figure}[t]
 \centering
 \includegraphics[width=0.92\textwidth]{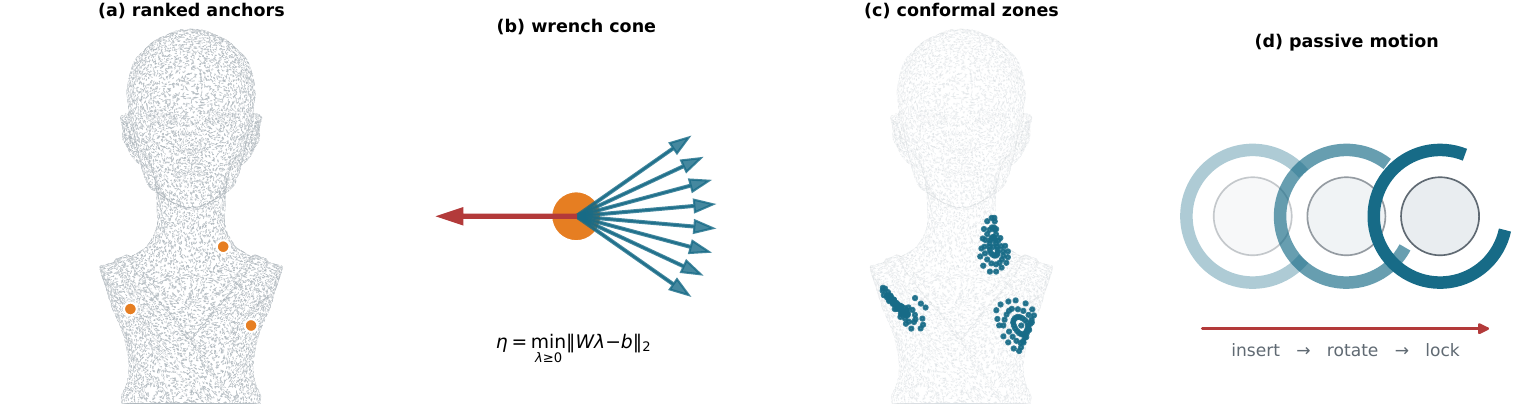}
 \caption{Attempt A4 illustration from sampled facets to ranked anchors, conformal zones, and passive insert--rotate--lock motion.}
 \label{fig:contactmotion}
\end{figure}

\subsection{Directional FEA and Protected Topology Optimization}

For $N_c$ contact zones, safety factor $s=2.5$, and dynamic factor $d=3$, the total structural load is $dsmg=7.5mg$. Each sample force $\vect F$ is distributed to eight nearby nodes using $w_j\propto1/\max(r_j,0.2h)$ and $\sum_jw_j=1$.

\begin{proposition}[Resultant preservation]
The nodal transfer preserves every applied resultant force.
\end{proposition}
\begin{proof}
$\sum_jw_j\vect F=(\sum_jw_j)\vect F=\vect F$; linearity extends the identity to all samples and zones.
\end{proof}

The approximately 3~mm mesh uses trilinear eight-node hexahedra, fixed mounting bosses, and $2^3$ Gauss integration:
\begin{equation}
 \mat K_e=\int_{\Omega_e}\mat B_e^{\mathsf T}\mat D\mat B_e\,d\Omega,
 \qquad \mat K\vect u=\vect f.
 \label{eq:fea}
\end{equation}
$\mat D=\mat S^{-1}$ is transversely orthotropic, with $S_{11}=1/E_x$, $S_{22}=1/E_y$, $S_{33}=1/E_z$, reciprocal Poisson coupling, and directional shear terms. Nominal PLA values are $E_x=E_y=1500$~MPa, $E_z=975$~MPa, 12/6~MPa XY/Z allowable stress, and a 1~mm displacement limit. Center-point stress recovery reports
\begin{multline}
 \sigma_{\rm vm}=\big[\tfrac12((\sigma_x-\sigma_y)^2+(\sigma_y-\sigma_z)^2
 +(\sigma_z-\sigma_x)^2)\\
 +3(\tau_{xy}^2+\tau_{yz}^2+\tau_{zx}^2)\big]^{1/2}.
 \label{eq:vm}
\end{multline}
Using this scalar with the lower directional allowable is a conservative digital screen, not a calibrated orthotropic failure criterion or physical certification.

With protected set $\mathcal P$, the SIMP problem is
\begin{align}
 \min_{\vect\rho}\;&C=\vect f^{\mathsf T}\vect u,\quad
 \mat K(\vect\rho)\vect u=\vect f,\quad N^{-1}\sum_e\rho_e\leq\bar\rho,\\
 &\rho_e=1\ (e\in\mathcal P),\qquad
 \mat K_e(\rho_e)=[\epsilon+(1-\epsilon)\rho_e^3]\mat K_e^0,
 \label{eq:simpopt}
\end{align}
where $\bar\rho=0.82$, $\epsilon=10^{-6}$, and $\mathcal P$ contains the mount, neck, contacts, stops, functional shell, and outer skin. A 5~mm density filter and optimality-criteria updates are followed by thresholding, minimum-radius void filtering, mount-component restoration, and post-topology FEA.

\begin{proposition}[Compliance monotonicity]
For positive-definite free-DOF stiffness, increasing an element density cannot increase compliance.
\end{proposition}
\begin{proof}
Equilibrium differentiation yields
$\partial C/\partial\rho_e=-\vect u_e^{\mathsf T}(\partial\mat K_e/\partial\rho_e)\vect u_e\leq0$ because
$\partial\mat K_e/\partial\rho_e=3(1-\epsilon)\rho_e^2\mat K_e^0$ is positive semidefinite.
\end{proof}

\begin{proposition}[Voxel containment]
The post-processed voxel design $\Omega_{\rm top}^h$ is a subset of the collision-approved preform $\Omega_{\rm pre}^h$.
\end{proposition}
\begin{proof}
Topology post-processing only removes material; restored or morphologically grown voxels are intersected with $\Omega_{\rm pre}^h$, and mount-hole recarving subtracts material. Hence $\Omega_{\rm top}^h\subseteq\Omega_{\rm pre}^h$.
\end{proof}
Surface extraction still requires the separate triangle-mesh collision audit.

\begin{figure}[t]
 \centering
 \includegraphics[width=0.92\textwidth]{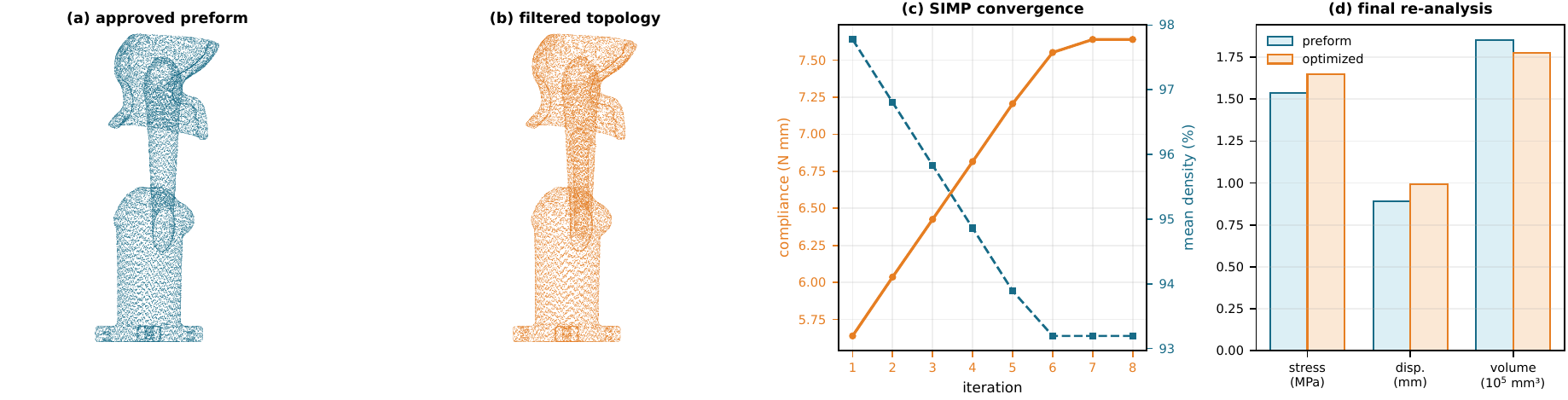}
 \caption{Attempt A4 approved preform, protected-domain SIMP result and convergence, followed by structural re-analysis.}
 \label{fig:structural}
\end{figure}

%% file: 4_Results.tex
\section{Experimental Setup and Results}
\label{sec:experiments}

\subsection{Dataset and Protocol}

Four completed dashboard attempts (A1--A4) were archived under \texttt{paper/scangrip\_attempts}. Each bundle contains the source object geometry, selected gripper, preform, topology result, contact and robustness evidence, robot trajectory, collision audit, FEA reports, and release manifest. Figure~\ref{fig:mechanisms} links each object to the mechanism selected from its accessible retention geometry and planned insertion motion. These are four documented design cases, not a statistical estimate of success over an object population.

\begin{figure}[h!]
 \centering
 \includegraphics[width=0.90\textwidth]{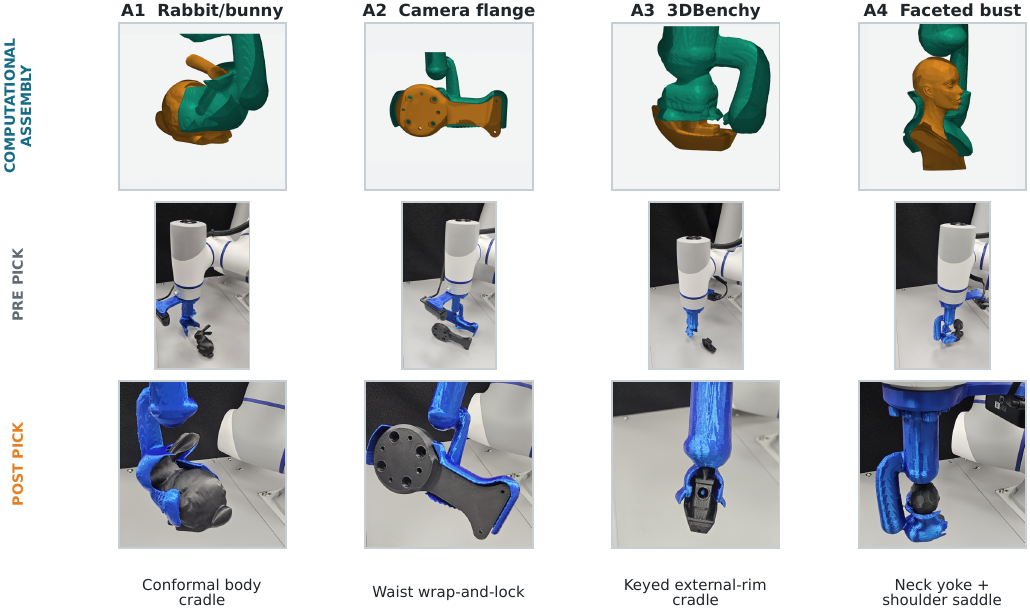}
 \caption{Object-specific passive mechanisms in the four archived attempts. The upper row shows each computational assembly, the middle row shows the robot immediately before pickup, and the lower row shows the printed gripper holding the object after pickup. The photographs document qualitative correspondence, not instrumented physical validation.}
 \label{fig:mechanisms}
\end{figure}

All attempts used a 0.20~kg declared computational load, the DOBOT CR10A with ISO 9409-1 adapter, the Creality K1 Max profile, and nominal PLA properties. Common search settings were 600 requested surface seeds (116--545 remained usable), 1,200 triplets, 12 ranked candidates, 160 local-fit trials, and 75 wrench-uncertainty trials. Geometry used a 0.55~mm manufacturing grid and approximately 3~mm FE elements. Topology settings were $p=3$, a 2~mm optimization grid, approximately 5~mm filtering, 0.82 target density, at most 18 iterations, a 4~mm minimum void radius, and a 3.6~mm protected product wall. All four optimizations converged in eight iterations. The ICP pathway is implemented for physical capture but was not evaluated here as a pose-accuracy experiment.

Table~\ref{tab:structuralmetrics} reports the structural records. Arrows compare the approved preform with the reconstructed, post-topology mesh; the final FE fraction is reported instead of treating target density as guaranteed exported-volume reduction.

\begin{table}[t]
\caption{Directional FEA and topology metrics from the four archived attempts. Arrows denote preform $\rightarrow$ post-topology values; $\rho_f$ is final thresholded FE material.}
\label{tab:structuralmetrics}
\centering
\fontsize{6.8}{7.8}\selectfont
\setlength{\tabcolsep}{2.4pt}
\begin{tabular}{@{}llrrrr@{}}
\toprule
Run & Object & FE elements & $\sigma_v$ (MPa) & $u$ (mm) & $\rho_f$ (\%)\\
\midrule
\input{generated_attempt_rows.tex}
\end{tabular}
\end{table}

\subsection{End-to-End Outcomes}

Table~\ref{tab:evidence} reports margins and worst-case quantities rather than repeating saturated acceptance counts.

\begin{table}[t]
\caption{Object-specific audit margins from the archived attempts. Stress fit gives retained-zone trials under an enlarged $\pm3$~mm, $\pm5^\circ$ pose box. These deterministic screen outputs are not physical success probabilities.}
\label{tab:evidence}
\centering
\fontsize{6.4}{7.4}\selectfont
\setlength{\tabcolsep}{1.8pt}
\begin{tabular}{@{}llrrrrrl@{}}
\toprule
& & \multicolumn{2}{c}{Local-fit audit} & \multicolumn{2}{c}{Wrench audit} & \multicolumn{2}{c}{Outcome}\\
\midrule
Run & Object & $m_{\min}$ (mm) & Stress fit & $r_{w,\max}$ & $c_{w,\max}$ & Sweep & Physical\\
\midrule
\input{generated_evidence_rows.tex}
\end{tabular}
\end{table}

Here $m_{\min}$ is the minimum retained-zone margin in the nominal $\pm0.8$~mm, $\pm1^\circ$ fit audit; $r_{w,\max}$ is the worst wrench residual against the 0.24 limit; and $c_{w,\max}$ is the largest normalized contact-load concentration. In the deliberately wider pose box, A1--A4 retained 134, 31, 81, and 29 of 160 trials, respectively. All four still cleared their 233 sampled runtime-sweep frames, baseline and post-topology FEA, and watertight reconstruction, but these are bounded digital gates rather than estimates of physical reliability. A1 uses broad, distributed rabbit-body support; A2 locks around a persistent flange waist; A3 keys below external deck and cabin rims rather than entering the hull; and A4 combines neck capture with a spatially separated shoulder saddle to resist roll. Thus the repeated C-wrench label in A2 and A4 denotes a family, while its object-conformal retention surfaces remain case-specific.

The finite runtime sweep covers the table, robot links, installed tool stack, generated mesh, selected object, attachment, placement, unlock, and retract using the recorded capsule-plus-visual-bounds collision model. It is a reproducible digital gate, not evidence of controller tracking or hardware safety. Likewise, nominal material properties, absent printer coupons, and the absence of instrumented fit, $3\times$ load, shake, and 100-cycle records keep every release manifest at \emph{digital screening passed, physical validation pending}.

\subsection{Structural Evidence}

Every preform and post-topology solve passed its nominal stress and displacement limits. Across A1--A4, maximum von Mises stress changed from 0.501--1.537~MPa before optimization to 0.501--1.647~MPa after reconstruction, while maximum displacement changed from 0.061--0.914~mm to 0.061--0.991~mm. The final FE material fractions of 92.0--97.6\% exceed the nominal 0.82 target because mounts, necks, contact surfaces, stops, functional shells, and outer skins are protected. The four results therefore support structural screening and domain containment, but do not justify a general mass-saving or physical-strength claim.

%% file: generated_attempt_rows.tex
% Generated from paper/scangrip_attempts evidence; do not edit by hand.
A1 & Rabbit/bunny & 5184 $\rightarrow$ 5066 & 0.791 $\rightarrow$ 0.791 & 0.299 $\rightarrow$ 0.301 & 92.0 \\
A2 & Camera flange & 5421 $\rightarrow$ 5315 & 0.887 $\rightarrow$ 0.950 & 0.914 $\rightarrow$ 0.961 & 94.8 \\
A3 & 3DBenchy & 7503 $\rightarrow$ 7503 & 0.501 $\rightarrow$ 0.501 & 0.061 $\rightarrow$ 0.061 & 97.6 \\
A4 & Faceted bust & 8382 $\rightarrow$ 8115 & 1.537 $\rightarrow$ 1.647 & 0.893 $\rightarrow$ 0.991 & 93.1 \\

%% file: generated_evidence_rows.tex
% Generated from paper/scangrip_attempts evidence; do not edit by hand.
% Stress fit replays the deterministic audit at +/-3 mm and +/-5 deg (160 trials).
A1 & Rabbit/bunny & 2.347 & 134/160 & 0.000 & 0.382 & 233 clear & pending \\
A2 & Camera flange & 0.452 & 31/160 & 0.006 & 0.432 & 233 clear & pending \\
A3 & 3DBenchy & 1.669 & 81/160 & 0.168 & 0.478 & 233 clear & pending \\
A4 & Faceted bust & 1.101 & 29/160 & 0.144 & 0.602 & 233 clear & pending \\

%% file: 5_conclusion.tex
\section{Discussion and Limitations}
\label{sec:discussion}

The four attempts support a staged co-design rather than a single grasp score. Point-wrench analysis is inexpensive enough for early screening, while distributed conformal geometry and robot motion decide whether a promising anchor set can become a usable tool. The rabbit, flange, boat, and bust mechanisms illustrate how one search framework yields different retention surfaces and insertion constraints. Structural optimization is similarly safest after fit surfaces, flange, motion corridor, and keep-outs are fixed. Because these protected regions occupy much of a compact gripper, target density is not promised mass reduction; exported volume and post-FEA response are the relevant outputs.

The main limitation is the absence of physical validation. The rigid Coulomb and small-perturbation models do not identify true friction, compliance, wear, print bias, or retention force. The robot records include a finite runtime scene sweep, but its capsule-plus-visual-bounds abstraction, controller tracking, calibration uncertainty, and cable motion remain unevaluated; the implemented selected-mesh ICP also lacks a calibrated pose-error dataset. The linear orthotropic FEA uses nominal rather than coupon-calibrated properties and omits nonlinear contact, creep, fatigue, bolt preload, and print defects. Sparse/nonmanifold scans and voxel resolution also affect accessibility and wall gates. Finally, six rigid passive families do not cover compliant, suction, magnetic, jamming, underactuated, or multi-material mechanisms.

The immediate next steps are therefore coupon-calibrated anisotropic material models, physical fit/load/shake/cycle experiments, and replacement of approximate robot-link bounds by fully calibrated triangle-mesh collision checking. Measured dimensional and force data can then replace uniform fit allowances and calibrate the relation between wrench residual, zone margin, and holding force.

\section{Conclusion}

The proposed framework joins deterministic contact screening, six explainable passive mechanisms, conformal geometry, robot-aware insertion, directional H8 finite elements, and protected SIMP topology optimization in one traceable process. The analysis proves preservation of applied resultant force, monotonic compliance sensitivity, and containment of post-processed topology voxels. Four archived dashboard attempts connect each selected object and mechanism to its computational fit, wrench, runtime-sweep, FEA, topology, and release evidence. All pass the recorded digital gates, but remain prototypes until calibrated material and physical fit, load, shake, and cycle tests pass.